\documentclass[article]{springer}

\usepackage{graphicx}
\usepackage{caption}
\usepackage{enumerate}
\usepackage{algorithm}
\usepackage{xcolor}
\usepackage{algorithmic}
\usepackage{url}

\usepackage{amssymb} 
\usepackage{amsmath}

\begin{document}

\title{ Cascading Randomized Weighted Majority: A New Online Ensemble Learning Algorithm}

\author{Mohammadzaman Zamani \and  Hamid Beigy  \and Amirreza Shaban}
\institute{Department Computer Engineering , Sharif University of Technology, Tehran, Iran\\
Department Computer Science and Engineering,  Ohio State University, USA}



\maketitle

\begin{abstract}
With the increasing volume of data in the world, the best approach for learning from this data is to exploit an online learning algorithm. Online ensemble methods are online algorithms which take advantage of an ensemble of classifiers to predict labels of data. Prediction with expert advice is a well-studied problem in the online ensemble learning literature. The Weighted Majority algorithm and the randomized weighted majority (RWM) are the most well-known solutions to this problem, aiming to converge to the best expert. Since among some expert, The best one does not necessarily have the minimum error in all regions of data space, defining specific regions and converging to the best expert in each of these regions will lead to a better result. In this paper, we aim to resolve this defect of RWM algorithms by proposing a novel online ensemble algorithm to the problem of prediction with expert advice. We propose a cascading version of RWM to achieve not only better experimental results but also a better error bound for sufficiently large datasets.
\end{abstract}


\section{Introduction}
Supervised learning algorithms are provided with instances already classified. In these algorithms, each instance has a label, identifying the class that instance belongs to. In supervised algorithms, the goal is to extract a general hypothesis from a number of labeled instances in order to make predictions about the unlabeled data.  Every learning algorithm uses a number of assumptions, therefore, it performs well in some domains; while it does not have appropriate performance in others \cite{Schaffer93}. As a result, combining classifiers is proposed as a new trend to improve the classification performance \cite{Bauer99}. The paradigm of prediction with expert advice is concerned with converging to the best expert among the ensemble of classifiers with a small misclassification rate during the operation of the algorithm. This concept has been studied extensively in the theoretical machine learning literature \cite{Chernov10supermartingales}, \cite{Bianchi06predictionlearning}, and attracts a lot of attentions in practice as well \cite{Vovk_predictionwith,Vovk95agameofprediction}.

Ensemble methods show outstanding performance in many applications, such as spam detection \cite{Saeedian08spamdetection,Zhen06anapproachtospam,Neumayer06clusteringbased}, intrusion detection \cite{Mukkamala05intrusiondetection,Fan02ensemblebased,Cherbrolu05featurededuction}, object tracking \cite{Avidan05ensembletracking,Tian07onlineensemble}, and feature ranking \cite{sb13a}. 
 Bagging \cite{Breiman96} and Boosting \cite{Freund96Experiments,Ross96baggingboosting}, which are well-known ensemble methods, rely on a combination of relatively weak models to achieve a satisfactory overall performance than their weak constituents. Online Bagging and Online Boosting\cite{Oza05} are also proposed to handle situations when the entire training data cannot fit into memory, or when the data set 
is of stream nature \cite{Oza05,sb11,hab12a,hab12b,hab13}. While the above-mentioned methods consist of weak learners, the mixture of experts algorithms selects among some learners that are experts in a specific input region \cite{Jacobs91adaptivemixtures,Jordan94hierarchicalmixtures}. Classification of data sequences is also the topic of recent research in the machine learning community. As a pioneering model of online 
ensemble learning, prediction with expert advice was first introduced in \cite{Desantis88learningprobabilistic,Littlestone94,Vovk90aggregatingstrategies} and 
recent investigations in this area lead to outstanding novel methods \cite{Bianchi05improvedsecondorder}.

Predicting with Expert Advice problem has the primary goal of predicting the label of data with an error rate close to that of the best expert. A 
simple and intuitive solution to this problem is Halving. Weighted Majority(WM) and its randomized version called randomized weighted majority  (RWM) are the most well-known solutions to 
this problem and presented in \cite{Littlestone94}. These algorithms are based on Halving, but have a better mistake bound dependent on the number of the 
experts and the error rate of the best expert. Another approach to this problem is Exponential Weighted Average(EWA) \cite{Bianchi06predictionlearning}, which is fundamentally very 
similar to RWM. Instead of using only zero-one loss function as is used in RWM, it exploits a convex loss function, and instead of finding a mistake bound, EWA 
obtains a regret bound.

All the above mentioned methods are based on the definition of the best expert. The best expert is the expert with the minimum average error during the 
experiment. Therefore, intuitively, the best expert does not necessarily have both the highest true negative rate and the highest true positive rate in the 
experiment. Our experiments reveal the fact that finding the best classifier for positive and negative output regions separately and predicting 
based on them, leads to a significant improvement in the performance of the classification.

In this paper, we propose a simple and effective method to find the experts that have the lowest false positive and the lowest false negative rates besides the overall 
best expert, simultaneously. The proposed method is called \emph{cascading randomized weighted majority} (CRWM), and presents a cascade version of RWMs to find these best experts. 
Theoretically, we show that CRWM converges to the best experts with a mistake bound tighter than that of RWMs in exposure to sufficient 
number of data points. Practically, our experiments on a wide range of well-known datasets support our contribution and show outstanding performance in 
comparison to not only RWM methods, but also some other well-known online ensemble methods. While we introduce the cascading method based on RWM, considering the 
similarities between RWM and EWA, CRWM can be similarly applied to EWA as a framework.

The rest of this paper is organized as follows: In the next section, online ensemble learning as an approach to the problem of predicting with 
expert advice is discussed. In section 3, the proposed algorithm and its mistake bound is presented. Section 4 evaluates the proposed algorithm and 
compares the experimental results to several other online ensemble methods. Finally, the conclusion is presented in section 5.

\section{Related Work}
In nearly all online learning algorithms for classification problem, there is a common scenario which consist of these phases: First of all, 
the learner is given with an instance, then the learner assigns a label to the given instance, and at the end the correct label of that instance is given to the 
learner; moreover, the learner learns this new labeled data to increase its performance. In the following we define predicting with 
expert advice problem which exploits a similar scenario, and Randomized Weighted Majority algorithm as one of its well-known solution.

\subsection{predicting with expert advice}
Let us consider a simple intuitive problem from \cite{Blum98}; a learning algorithm is given the task of predicting the weather each day that if it will rain today or not.
 In order to make this prediction, the thought of several experts is given to the algorithm. Each day, every expert says yes or no to this question and 
 the learning algorithm should exploit this information to predict its opinion about the weather. After making the prediction, the algorithm is told how 
 the weather is today. It will be decent if the learner can determine who the best expert is and predict the best expert's opinion as its output. 
 Since we do not know who the best expert is, our goal instead would be performing nearly as the best expert's opinion so far. It means a good learner should guarantee 
 that at any time, it has not performed much worse than none of the experts.  An algorithm that solves this problem is consist of the following stages. 
First, it receives the predictions of the experts. Second, Makes its own perdition and third, Finally, it is told the correct answer.

\subsection{Randomized weighted majority}
Weighted Majority algorithm and its randomized version are the most famous solutions of predicting with expert advice problem. The Randomized Weighted 
Majority, which is the fundamental part of the proposed algorithm, has several base classifiers(expert) and each classifier has a weight factor. Every time a new 
instance is received, each of these base classifiers predicts a label for that instance, and the algorithm decides a label based on these predictions and the weight 
factors of the classifiers. Whenever the true label of that instance arrives, the weight factors of the classifiers should be updated in a way that each 
classifier that predicts a wrong label would be penalized by a constant factor. This algorithm is proven that converges to the best classifier among all the 
base classifiers.  Algorithm 1 describes the pseudo code of the randomized weighted majority algorithm.

\begin{algorithm}
\caption{Randomized weighted majority \cite{Littlestone94}}
\begin{algorithmic}

 \REQUIRE $w_i \gets 1$ $\vee$ \newline
	$x_{ij}$ $\gets$ prediction of $i^{th}$ expert on $j^{th}$ data $\vee$  \newline
	$n \gets$ number of experts $\vee$  \newline
	$C_j \gets$ correct label of $j^{th}$ data $\vee$ \newline
	$N \gets$ number of data $\vee$ \newline
	$ \beta \gets$ the penalty parameter
 \ENSURE $y_j$ = output label of $j^{th}$ data

\FOR{$j = 1 \to N$}
\STATE W = $\sum_{i} w_i$ 
\STATE $y_j = x_{ij}$, with probability $w_i/W$
\FOR{$i = 1 \to n$}
\IF{$x_{ij} \neq C_j$}
\STATE $w_i \gets w_i \times \beta $
\ENDIF
\ENDFOR
\ENDFOR
\end{algorithmic}
\end{algorithm}

It has been shown that, on any trial sequence, the expected number of Mistakes (M) made by randomized weighted majority Algorithm satisfies the following inequality 
\cite{Littlestone94}:
\begin{equation} 
	M \le \frac{m \ln(1/\beta) + \ln n}{(1-\beta)},
\end{equation}
where m is the number of mistakes made by the best expert so far and $\beta$ is a constant penalizing value.

\section{ The cascading randomized weighted majority algorithm }

When the data size becomes too large, the randomized weighted majority (RWM) algorithm tends to decide according to the best expert's opinion. It takes time for the algorithm to converge to the best expert, and the learner may make more mistakes compared to the best one. However, when the best expert is discovered by the learner, one 
should be sure that the algorithm did not make many mistakes more than the best one, and we can say that it predicts whatever the best expert says from 
now on.

In this section, we propose an online learning algorithm which its main idea is to define more than one best expert, each for a number of data instances. In fact, the algorithm tries to find the best experts, and for every new data instance, decides which expert is the most suitable one, and predicts according to the opinion of that expert.

Each expert is actually a classifier. As we studied numerous classifiers, we observed that they often do not have low false positive (FP) rates at the same time as having low false negative (FN) rates. Figure \ref{fig:overflow} shows a two-class dataset and three different linear classifiers. Although, $C_o$ has lowest error rate its FP rate and also FN rate are not the best among these three classifiers. As it is shown in the figure, $C_p$ has lower FP rate than $C_o$ and $C_n$ has lower FN rate than $C_o$, either. As a result, instead of looking for the expert with the lowest error rate, we look for experts with lowest FP and FN rates leading us to define three best classifiers.

\begin{figure}[ht!]
\centering
\includegraphics[]{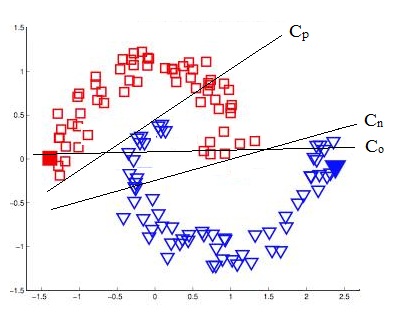}
\caption{The different classifiers for the given data set}
\label{fig:overflow}
\end{figure}

Finding the overall best expert in an ensemble of online classifiers is the goal of Weighted Majority algorithm; However, finding the best  positive expert and the best negative expert simultaneously is 
still a problem and we propose the following algorithm to solve this problem. 

\begin{figure}[ht!]
\centering
\includegraphics[width=128mm]{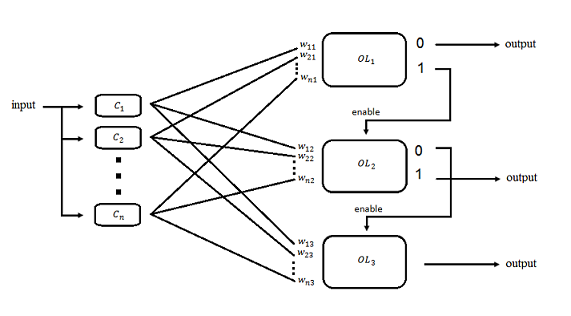}
\caption{The structure of the CRWM algorithm}
\label{fig:crwm}
\end{figure}

As it is shown in the Figure \ref{fig:crwm}, the proposed algorithm has $n$ base classifiers, and three online learners with RWM mechanism that exploit these classifiers. In other words, 
there are $3 \times n$ weight factors, $n$ weight factors for each learner. For every new data instance, each of these $n$ base 
classifiers predicts a label. Using the corresponding weight factors, these predictions are given to the learners in order for them to make their predictions. Since the first algorithm is responsible for predicting negative labels, If this algorithm predicts the label as negative, the label is set to negative, otherwise, the second algorithm will be applied. If the second algorithm predicts the label as positive, the output label will be positive, otherwise the third algorithm will be applied and the output of this algorithm would be the output label for the instance. 

Although the described scenario states the algorithm for two class datasets, the scenario is so similar for multi-class datasets. For a multi-class dataset, the algorithms would still use $n$ base classifiers, but it needs $L+1$ online learners, where $L$ is the number of different classes. One for each output class besides 
a final learner. So, in an $L$ class dataset the algorithm would have $(L+1) \times n$ weight factors. The rest of algorithm is same as the algorithm for two class datasets. For simplicity; in what follows, we explain about two class datasets. However, all the following; including the proof of the algorithm, still stands for multi-class datasets.

Whenever the correct label of the instance arrived, this correct label will be given only to the learner that produced the output label. Therefore, on every data instance, the weight factors of only one learner can be updated. As the learners are 
using RWM mechanism; in the specified learner, the weight factor of the classifiers that made the wrong prediction would be penalized by a constant factor $\beta$. 

Since the CRWM algorithm is using three RWM learners in three levels, and due to its cascading structure, we called it cascading randomized weighted majority. Algorithm 2 is the pseudo code of CRWM algorithm.

\begin{algorithm}
\caption{Cascading randomized weighted majority algorithm}
\begin{algorithmic}

 \REQUIRE $w_ij \gets 1$ $\vee$ \newline
	$x_{ijk}$ $\gets$ prediction of $i^{th}$ base classifier of $j^{th}$ learner on $k^{th}$ data $\vee$  \newline
	$n \gets$ number of experts $\vee$  \newline
	$C_j \gets$ correct label of $j^{th}$ data $\vee$ \newline
	$N \gets$ number of data $\vee$ \newline
	$ \beta \gets$ the penalty parameter
 \ENSURE $y_j$ = output label of $j^{th}$ data
 
\FOR{$j = 1 \to N$}
	\STATE $W_1 = \sum_{i} w_{i1}$ 
	\STATE $y1 = x_{i1k}$, with probability $w_{i1}/W_1$
	\STATE $W_2 = \sum_{i} w_{i2}$ 
	\STATE $y2 = x_{i2k}$, with probability $w_{i2}/W_2$
	\STATE $W_3 = \sum_{i} w_{i3}$ 
	\STATE $y3 = x_{i3k}$, with probability $w_{i3}/W_3$
	\IF{y1 = 0}
		\FOR{$i = 1 \to n$}
			\IF{$x_{ij1} \neq C_j$}
				\STATE $w_{i1} \gets w_{i1} \times \beta $
			\ENDIF
		\ENDFOR
	\ELSE \IF{y2 = 1}
		\FOR{$i = 1 \to n$}
			\IF{$x_{ij2} \neq C_j$}
				\STATE $w_{i2} \gets w_{i2} \times \beta $
			\ENDIF
		\ENDFOR
	\ELSE 
		\FOR{$i = 1 \to n$}
			\IF{$x_{ij3} \neq C_j$}
				\STATE $w_{i3} \gets w_{i3} \times \beta $
			\ENDIF
		\ENDFOR
	\ENDIF
	\ENDIF
\ENDFOR

\end{algorithmic}
\end{algorithm}

\subsection{The Mistake Bound}
In this section, we intend to find the bound on the number of mistakes which probably occurred by this algorithm in a series of predictions. We know that there is a 
mistake bound for RWM algorithm. First, we use the method thereby RWM mistake bound was calculated to find a mistake bound for CRWM and subsequently, 
show that this mistake bound is better than RWM's mistake bound, when the data size is large enough. 

\begin{theorem}On any sequence of trials, the expected number of mistakes ($M_j$) made by $j^{th}$ learner ($OL_1$, $OL_2$, $OL_3$) in cascading randomized weighted majority algorithm satisfies the following condition.
\begin{equation}
M_j \le  \frac {m_{kj} \times  \ln(1/ \beta) + \ln n} {(1- \beta)},  \hspace{2mm}\forall (j, k),
\end{equation}
where $m_{kj}$ is the number of mistakes made by the $k^{th}$ expert in $j^{th}$ learner so far and $n$ is the number of base classifiers(experts).
\end{theorem}

\begin{proof}Define $F_{ij}$ as the fraction of the total weight on the wrong answer at the $i^{th}$ trial on the $j^{th}$ learner($OL_1$, $OL_2$, $OL_3$), 
and let $M_j$ be the expected number of mistakes of the $j^{th}$ learner so far. So after the $t^{th}$ trial, we would have $ M_j = \sum_{i=1}^{t}{F_{ij}}$.

On the $i^{th}$ instance that classified by the $j^{th}$ learner, the total weight of the $j^{th}$ learner where is defined by $W_j$ changes according to:
\begin{equation}
W_j \leftarrow W_{j} (1-(1-\beta)F_{ij}).
\end{equation}

Since when the data is classified by $j^{th}$ learner, we multiply the weights of experts in $j^{th}$ learner that made a mistake by $\beta$ and there is an 
$F_{ij}$ fraction of the weight on these experts. Regarding to the initial value of weight factor for each base classifier which is set to 1, and considering 
$n$ as the number of base classifiers, the final total weight for the $j^{th}$ learner is:
\begin{equation}
W_j = n \prod_{i=1}^t {(1-(1-\beta)F_{ij})}.
\end{equation}

Let $m_{kj}$ be the number of total mistakes of the $k^{th}$ base classifier in the $j^{th}$ learner so far, therefore its weight factor would be $\beta^{m_{kj}}$ at this time. Using the fact that the total weight must be at 
least as large as the weight of the $k^{th}$ classifier; for each value of j and k, we have:
\begin{equation}
n \prod_{i=1}^t {(1-(1-\beta)F_{ij})} \ge \beta^{m_{kj}}.
\end{equation}

Taking the natural log of both side we get:
\begin{equation}
\ln n + \sum_{i=1}^t \ln{(1-(1-\beta)F_{ij})} \ge m_{kj} \ln \beta
\end{equation}
\begin{equation}
- \ln n - \sum_{i=1}^t \ln{(1-(1-\beta)F_{ij})} \le m_{kj} \ln (1/\beta).
\end{equation}

Since, $\forall \hspace{1mm} 0<x<1 , - \ln(1-x) > x $, following equation will be obtained: 
\begin{equation}
- \ln n + (1-\beta) \sum_{i=1}^t {F_{ij}} \le m_{kj} \ln (1/\beta).
\end{equation}

Using $ M_j = \sum_{i=1}^t {F_ij}$, we conclude:
\begin{equation}
M_j \le \frac{m_{kj} \times \ln(1/\beta) + \ln n}{(1-B)},   \hspace{2mm}\forall (j, k).
\end{equation}

Which completes the proof of theorem. 
\end{proof}

Now we have a bound for the expected number of mistakes that every learner will do in a sequence of trials. Since for 
each instance only one of these three learners respondes and predicts the output label, in the following theorem we intend to find the total expected number 
of mistakes for the algorithm by aggregating the expected number of mistakes of these learners.

\begin{theorem} 
 On any sequence of trials, the expected number of mistakes made by cascading randomized weighted majority algorithm$(M_{CRWM})$ satisfies the following condition.
\begin{equation}
M_{CRWM} \le  \frac { \sum_{i=1}^3 m_i \times  \ln(1/ \beta) + 3 \ln n} {(1- \beta)},
\end{equation}
where $m_i$ is the number of mistakes made by the best expert of $i^{th}$ learner so far, and $n$ is the number of experts in each learner.
\end{theorem} 

\begin{proof}Since, each instance is classified by exactly one of the three learners, the expected number of mistakes  
made by cascading randomized weighted majority algorithm can be obtained with the following equation:
\begin{equation}
M_{CRWM} = \sum_{j=1}^3 M_j
\end{equation}

By using Theorem 1 and defining $k, k'$ and $k"$ as the index of the best expert of $OL_1, OL_2$ and $OL_3$, respectively, we have:
\begin{equation}
M_{CRWM} \le  \frac {(m_{k1} + m_{k'2} + m_{k"3}) \times  \ln(1/ \beta) + 3 \ln n} {(1- \beta)}.
\end{equation}

For convenience we define $m_i$ as the number of mistakes made by the best expert of $i^{th}$ learner, So we would have:
\begin{equation}
M_{CRWM} \le  \frac { \sum_{i=1}^3 m_i \times  \ln(1/ \beta) + 3 \ln n} {(1- \beta)}.
\end{equation}

Which completes the proof of the theorem. 
\end{proof}

Now we need to know how suitable is this bound. For this purpose, we will compare the bound of CRWM algorithm with  the bound of RWM algorithm in the following theorem.

\begin{theorem} The mistake bound for CRWM algorithm is better than RWM mistake bound, when the data size is going to be large.
\end{theorem} 

\begin{proof} 
We know that the expected number of mistake for RWM Algorithm is bounded by the following inequality:
\begin{equation} 
M_{RWM} \le \frac{m \ln(1/B) + \ln n}{(1-B)} = Bound_{RWM}.
\end{equation}

We have already defined $Bound_{RWM}$ in the above equation In addition; using  eq.(12), we define $Bound_{CRWM}$ as follows:
\begin{equation}
Bound_{CRWM} = \frac {(m_{k1} + m_{k'2} + m_{k"3}) \times  \ln(1/ \beta) + 3 \ln n} {(1- \beta)}.
\end{equation}

Now, it is needed to show that by increasing the number of incoming data, $Bound_{CRWM}$ is less than $Bound_{RWM}$, so we should have:
\begin{equation}
Bound_{CRWM} - Bound{RWM} < 0.
\end{equation}

Using the eq.(14) and the eq.(15) we derive the following inequality:
\begin{equation}
\frac {(m_{k1} + m_{k'2} + m_{k"3}) \times  \ln(1/ \beta) + 3 \ln n} {(1- \beta)} - \frac{m \ln(1/B) + \ln n}{(1-B)} < 0,
\end{equation}
where $m$ is the number of overall errors of the best expert. Lets suppose the index of the overall best expert is $p$. So, the number of mistakes made by the 
overall best expert in $OL_1$ is defined by $m_{p1}$ and similarly $m_{p2}$ and $m_{p3}$ are the number of mistakes of the overall best expert in $OL_2$ and $OL_3$, respectively. By using these definitions and considering the fact that $ m = m_{p1}+m_{p2}+m_{p3}$, we can rewrite the above inequality as follows:
\begin{equation}
\frac {(m_{k1}-m_{p1} + m_{k'2}-m_{p2} + m_{k"3}-m_{p3}) \times  \ln(1/ \beta) + 2 \ln n} {(1- \beta)}  <0.
\end{equation}

Since $k$ is the index of best expert in $OL_1$, obviously $m_{k1} \le m_{p1}$. This also stands for $m_{k'2}$ and $m_{k"3}$, so we have $m_{k'2} \le m_{p2}$ and $m_{k"3} \le m_{p3}$. 
By considering the region of instances that has been classified by $j^{th}$ learner, we define $K_j$ as the number of instances in this region and  
$X_{ij}$ as the error rate of the $i^{th}$ expert in $j^{th}$ learner in the specified region of instances. So we would have:
\begin{equation}
m_{ij} = K_j * X_{ij},  \hspace{2mm} \forall (i,j).
\end{equation} 

As we mentioned earlier as our main hypothesis, the overall best expert does not have the best error rate in both positive and negative regions. So, 
following inequality holds. 
\begin{equation}
X_{k1} < X_{p1} \hspace{2mm} or \hspace{2mm} X_{k'2} < X_{p2}.
\end{equation}

Without loss of generality,  suppose $X_{k1} < X_{p1}$. So, we have $X_{k'2} = X_{p2}$ and $X_{k"3} = X_{p3}$. Using these facts and also eq.(19) we can 
rewrite the eq.(18) as follows:
\begin{equation}
\frac {K_1\times(X_{k1}-X_{p1}) \times  \ln(1/ \beta) + 2 \ln n} {(1- \beta)}  <0.
\end{equation}

By some algebraic simplification, we obtain:
\begin{equation} 
K_1 \times (X_{p1}-X_{k1}) > \frac {2\ln n}{ln(1/B)}.
\end{equation}

In the above inequality, the only parameters that will be raised with increasing the size of dataset is $K_1$ and the other parameters, which are 
given bellow, are constants and limited.
\begin{equation} 
X_{p1}-X_{k1} = C_0 > 0
\end{equation}
\begin{equation} 
\frac {2\ln n}{ln(1/B)}= C_1 \ge 0.
\end{equation}

Using the above definitions we can rewrite eq.(22) as follows:
\begin{equation} 
K_1 \times C_0 > C_1.
\end{equation}

Since $K_1$ is increasing with increase of the number of data instances, hence by increasing the number of data instances, the above equation would be true which means the mistake 
bound of CRWM would be better than the mistake bound of RWM, which completes the proof of the theorem.
\end{proof}

\section{Experimental Results}

In this section, we compared the classification performance of the proposed method with randomized weighted majority, Online bagging and Online boosting on 
14 datasets.  These datasets are from the UCI Machine Learning  repository \cite{BacheLichman2013} to evaluate different aspects of the algorithms. The different characteristics 
of these datasets are shown in table 1. The number of instances in these datasets vary from 208 to 490000 and the number of attributes vary from 4 to 64. In 
order to show the effectiveness of the proposed method, some multi-class datasets are chosen as well as two-class datasets. All the four algorithms are 
implemented in JAVA using MOA framework. In all the implementations, we exploit naive Bayes as the base classifier algorithm, due to the fact that it
 is highly fast and can easily updated algorithm as well. In \cite{Oza05} the best performance of online bagging and online boosting achieved using 100 number of
 base classifiers. To have fair comparison we have used the same number of naive Bayes base classifiers in all of these methods.
 
\begin{table}[]
\centering
\caption{Datasets used for evaluation of algorithms}
\begin{tabular}{|l|c|c|c|} \hline
\bfseries Dataset & \bfseries  \#Examples & \bfseries \#Features & \bfseries \#Classes \\ \hline 

Sonar		&	208	&	60	&	2\\
Ionosphere		&	351	&	34	&	2 \\
Balance scale		&	625	&	4	&	3\\
Breast cancer		&	699	&	9	&	2\\
Diabetes		&	768	&	8	&	2\\
German credit		&	1000	&	20	&	2\\
Chess(rockvspawn)		&	3196	&	36	&	2\\
Spambase		&	4601	&	57	&	2 \\
Optdigits		&	5620	&	64	&	10\\
Mushroom		&	8124	&	22	&	2\\
Pendigit		&	10992	&	16	&	10\\
Nursary		&	12960	&	8	&	5\\
Letter recognition		&	20000	&	16	&	26\\
KDDCup99		&	490000	&	42	&	23\\

\hline \end{tabular}
\end{table}

There is only one parameters in the proposed algorithms, which is the penalty 
parameter ($\beta$) in both RWM and CRWM. We use $\beta = 0.5$ in all the experiments. Fig. 3 illustrates how the accuracies of RWM 
and CRWM algorithms depend on $\beta$. In this figure, the results of an experiment on only 4 datasets are shown. These results show that 
using this value for $\beta$ the algorithms often do near the best performance.

\begin{figure*}[ht!]
\centering
\includegraphics[]{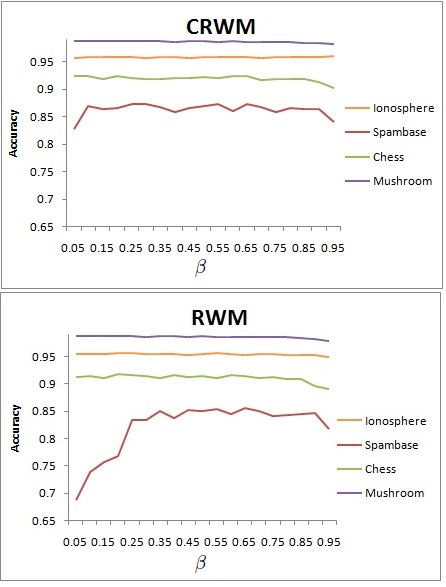}
\caption{Sensitivity of Accuracy to parameter $\beta$ in RWM and CRWM}
\label{overflow}
\end{figure*}

The idea of looking for the experts with lowest FP rate seems problematic. For instance, if the expert predicts everything as negative, FP rate of it would 
be zero and similarly, concern is there for lowest FN rate. In order to avoid such problems, we can exploit a biased estimation of FP rate instead of FP rate itself.
\begin{equation}
FP\hspace{1mm}Rate\hspace{1mm}Biased\hspace{1mm}Estimation = \frac {FP + n_p} { FP + TP + n},
\end{equation}
where, $n_p$ and $n$ are constant. For simplicity we can set $n_p$ to 1 and $n$ to 2. We also define the biased estimation of FN rate in the same way.

To determine the performance of the proposed algorithm, we compare the accuracy of CRWM algorithm with RWM, online bagging 
and online boosting. Table 2 shows the results of each of the algorithms on all the datasets. The performance measure used for comparison is the 
accuracy of the algorithms, which is the average accuracy of algorithms in 50 independent runs. The scenario which is used for evaluating the accuracy is the 
exact scenario of online learning; in which, every new instance is given to the algorithm and the algorithm makes its prediction. Then, the correct label 
would be given and will be compared to the output prediction of the algorithm.

\begin{table}[hbtp]
\centering
\caption{Accuracy of different algorithms}
\begin{tabular}{|l|c|c|c|c|l} \hline
\bfseries Dataset & \bfseries Online Bagging & \bfseries Online Boosting & \bfseries RWM &\bfseries CRWM \\ \hline 

Sonar	&	67.68	&	\textcolor{green}{68.69}	&	66.38	&	\textcolor{blue}{67.73}	\\  
Ionosphere	&	85.26	&	\textcolor{blue}{86.13}	&	82.23	&	\textcolor{green}{86.24}	\\
Balance scale	&	\textcolor{green}{83.52}	&	75.31	&	73.76	&	\textcolor{blue}{82.86}	\\
Breast cancer	&	\textcolor{green}{95.91}	&	94.03	&	93.97	&	\textcolor{blue}{95.77}	\\
Diabetes	&	\textcolor{blue}{74.28}	&	67.91	&	71.79	&	\textcolor{green}{74.43}	\\
German credit	&	\textcolor{green}{73.69}	&	67.94	&	72.95	&	\textcolor{blue}{73.55}	\\
Chess(rockvspawn)	&	86.81	&	90.88	&	\textcolor{blue}{90.96}	&	\textcolor{green}{91.37}	\\
Spambase	&	78.99	&	84.07	&	\textcolor{blue}{85.95}	&	\textcolor{green}{87.32}	\\
Optdigits	&	\textcolor{blue}{89.46}	&	87.51	&	89.13	&	\textcolor{green}{89.94}	\\
Mushroom	&	95.25	&	\textcolor{green}{99.71}	&	97.29	&	\textcolor{blue}{98.63}	\\
Pendigit	&	\textcolor{blue}{85.61}	&	86.23	&	85.33	&	\textcolor{green}{86.27}	\\
Nursary	&	\textcolor{blue}{89.88}	&	89.52	&	85.1	&	\textcolor{green}{89.96}	\\
Letter recognition	&	63.4	&	58.93	&	\textcolor{blue}{63.89}	&	\textcolor{green}{64.53}	\\
KDDCup99 &	\textcolor{blue}{96.86}	&	96.41	&	96.51	&	\textcolor{green}{97.67}	\\

\hline \end{tabular}
\end{table}

As it has shown in table 2, the results of CRWM are dominant in most of the datasets. Another fact is that in the remaining datasets it has the second best performance. 
These stable results can be considered as a great power of the proposed method. 

Another result that comes from table 2 is that the accuracy of CRWM is higher compared to RWM in all the datasets, which means our theoretical bound 
is fully supported by our experimental results. Although we have shown that the CRWM obtains more reasonable results when the number of data is increased, 
we can see even in datasets with a few number of instances that the CRWM algorithm has better accuracy compared to the RWM, which shows another power of the CRWM algorithm.

While in some data sets there is significant difference between accuracy of CRWM and other algorithms, in some others the difference is not very clear. 
So, we have arranged student's t-test to clarify this ambiguity. Table 3 shows the result of paired student's t-test between CRWM and the other three 
algorithms. 
The output of this test is P-value. A P-value below 0.05 is generally considered statistically significant, So the one who has better average is considered to have better results than the other.
 while one of 0.05 or greater indicates no difference between the groups. In this 
table, a blue cell shows there is no difference. Using the output P-value of this test we have arranged table3. In which, a red cell means CRWM is worse than the specified algorithm on specified dataset, while black cells; which 
cover most of the table (36/42 number of cells), show that CRWM is better in compared to that algorithm on that dataset. As it is clear, only 3 cells of the table are
red and just 6 cells are blue, which means the great performance of CRWM algorithm in compared to other algorithms and it confirms the results of table 2.

\begin{table}[hbtp]
\centering
\caption{Paired student's t-test beetween CRWM and other algorithms }
\begin{tabular}{|p{5cm}| p{1.8cm}| p{1.8cm} |p{1.8cm} |l c c c l} \hline
\bfseries Dataset & \bfseries CRWM vs. Bagging & \bfseries CRWM vs. Boosting & \bfseries CRWM vs. RWM  \\ \hline 

Sonar	&	\textcolor{blue}{draw}	&	\textcolor{blue}{draw}	&	win	\\
Ionosphere	&	win	&	\textcolor{blue}{draw}	&	win	\\
Balance scale	&	\textcolor{red}{lose}	&	win	&	win	\\
Breast cancer	&	\textcolor{red}{lose}	&	win	&	win	\\
Diabetes	&	\textcolor{blue}{draw}	&	win	&	win	\\
German credit	&	\textcolor{blue}{draw}	&	win	&	win	\\
Chess(rockvspawn)	&	win	&	win	&	win	\\
Spambase	&	win	&	win	&	win	\\
Optdigits	&	win	&	win	&	win	\\
Mushroom	&	win	&	\textcolor{red}{lose}	&	\textcolor{blue}{draw}	\\
Pendigit	&	win	&	win	&	win	\\
Nursary	&	win	&	win	&	win	\\
Letter recognition	&	win	&	win	&	win	\\
KDDCup99	&	win	&	win	&	win	\\
\hline 
OVERALL(win/\textcolor{blue}{draw}/\textcolor{red}{lose})&	9/\textcolor{blue}{3}/\textcolor{red}{2}	&	11/\textcolor{blue}{2}/\textcolor{red}{1}	&	13/\textcolor{blue}{1}/\textcolor{red}{0} \\

\hline \end{tabular}
\end{table}

Considering the cascading structure of CRWM and also its outstanding results, it seems that this excellence is obtained with the cost of 
more tunning time. However, the results of table 4 refuse this view.
Table 4 shows the running time required for each algorithm on every dataset. The results of this table show that there is no significant difference in 
running time between CRWM, RWM and Online Bagging,
Even though the number of weight factors in CRWM depend on the number of classes, it still 
exploits the same number of base classifiers as other algorithms, which is the main factor in running time of the algorithms. In addition, after receiving 
the correct label, CRWM gives it to only one of the learners to learn and justify its weights. These points of view indicate why there is no great difference 
between the running times of this algorithms. The only reason that makes a little difference in running times is the overhead of creating and using more weight factors, 

\begin{table}[hbtp]
\centering
\caption{Running times (sec.)}
\begin{tabular}{|l|c|c|c|c|} \hline
\textbf{Dataset} & \textbf{Online Bagging} & \textbf{Online Boosting} & \textbf{RWM} &\textbf{CRWM} \\ \hline 

Sonar		&	0.54	&	1.12	&	0.53	&	0.57	\\
Ionosphere		&	0.66	&	1.22	&	0.68	&	0.7	\\
Balance scale		&	0.1	&	0.17	&	0.1	&	0.16	\\
Breast cancer		&	0.24	&	0.34	&	0.23	&	0.27	\\
Diabetes		&	0.17	&	0.36	&	0.17	&	0.21	\\
German credit		&	0.34	&	0.72	&	0.4	&	0.44	\\
Chess(rockvspawn)		&	0.98	&	1.84	&	0.96	&	1.45	\\
Sambase		&	8.1	&	10.3	&	8.05	&	8.78	\\
Optdigits		&	50.6	&	83.1	&	49.1	&	54.8	\\
Mushroom		&	1.7	&	2.07	&	1.6	&	2.18	\\
Pendigit		&	25.8	&	40.6	&	26.1	&	28.4	\\
Nursary		&	1.23	&	5.61	&	1.28	&	2.27	\\
Letter~recognition		&	102	&	196	&	106	&	111	\\
KDDCup99		&	6674	&	9017	&	6232	&	6750	\\

\hline \end{tabular}
\end{table}

 Table 5 aims to confirm our main hypothesis about comparison of the values of best FP rate, best FN rate and best error rate among the base classifiers.
A green cell in the best FP rate and best FN rate columns means that the value is better than the corresponding error rate value, and existence of such cell
in every row is exactly what what we have assumed to be true in every groups of base classifiers.

\begin{table}[hbtp]
\centering
\caption{best FP, FN and error rate among the base classifiers}
\begin{tabular}{|l|c|c|c|l} \hline
\textbf{Dataset} & \textbf{best FP rate} & \textbf{best FN rate} & \textbf{best error rate}  \\ \hline 

Sonar	&	0.306	&	\textcolor{green}{0.222}	&	0.277	\\
Ionosphere	&	\textcolor{green}{0.077}	&	\textcolor{green}{0.083}	&	0.112	\\
Breast cancer	&	\textcolor{green}{0.011}	&	0.061	&	0.035	\\
Diabetes	&	0.293	&	\textcolor{green}{0.206}	&	0.238	\\
German credit	&	\textcolor{green}{0.201}	&	0.394	&	0.251	\\
Chess(rockvspawn)	&	0.098	&	\textcolor{green}{0.059}	&	0.085	\\
Spambase	&	\textcolor{green}{0.038}	&	\textcolor{green}{0.132}	&	0.133	\\
Mushroom	&	0.02	&	\textcolor{green}{0.001}	&	0.019	\\

\hline \end{tabular}
\end{table}

Table 6 aims to compare the mistake bound of CRWM and RWM with each other for the two-class datasets used in our experiments. It also compares these 
theoretical mistake bounds with the experimental results. In this experiment, the value of theoretical mistake bound is calculated using the results of 
table 5 and some other parameters that exist in corresponding formulas. As it is shown in table 6, the theoretical mistake bound is always greater than 
the corresponding experimental result, which confirms the accuracy of calculated mistake bound. In addition, when the size of datasets is small, 
the mistake bound of RWM is lower than the one of CRWM. However, in larger datasets the mistake bound of CRWM excels its rival, which is exactly the 
point that we have mentioned in theorem 3.
 
\begin{table}
\centering
\caption{Number of Mistakes in Theoretical Mistake Bound vs. Experimental Result}
\begin{tabular}{|l|c|c|c|c|}
\hline
 &\multicolumn{2}{c|}{\textbf{CRWM}}&\multicolumn{2}{c|}{\textbf{RWM}}\\
\cline{2-5}
\bfseries Dataset & \bfseries Mistake bound & \bfseries Result & \bfseries Mistake bound & \bfseries Result \\ 
\hline\hline

Sonar	&	103&	66	&	88.6	&	69	\\
Ionosphere	&	68.9	&	47	&	63.4	&	61	\\
Cancer	&	55.6	&	29	&	42.4	&	41	\\
Diabetes	&	270	&	195	&	260 &	214	\\
German credit	&	350	&	262	&	354	&	267	\\
Chess(rockvspawn)	&	376	&	273	&	381	&	286	\\
spambase	&	650	&	578	&	847	&	640	\\
mushroom	&	145	&	110	&	227	&	218	\\
\hline

\end{tabular}
\end{table}

\section{Conclusion and Future Works}

In this paper, we proposed a new online ensemble learning algorithm, called CRWM. It is shown that CRWM's mistake bound is 
better than that of RWM's when the size of the input is increased. In addition, the experimental results reveal that CRWM obtains a better accuracy compared 
to RWM with a wide range of input sizes. By carrying out several experiments, we have shown that this new algorithm outperforms other online ensemble learning 
algorithms. It usually acquires the best performance or the second best performance among these algorithms, indicating its superiority among them.

In this study, we did not address imbalanced datasets. However, the structure of CRWM that focuses on each class separately, provides us a powerful 
facility in defining different cost functions on each class. Besides, it provides us with a means to change the order of online learners. For instance,
for classes that are needed to have more true positive we can move their corresponding learners to the top of the structure. Inversely, whenever the false 
positive of a specified class is significant we can move the corresponding learner to the bottom of CRWM structure. This dynamic structure is a powerful feature that distinguishes CRWM from other similar algorithms.

Clearly, using different base classifiers for each learner may lead to a better accuracy. In addition, utilizing one-class classifiers as base classifiers 
would cause great effects on accuracy. However it is needless to say that, using more base classifiers lead to an increase in running time of the algorithm. 
While, using just different weight factors, as we did in CRWM, does not affect it so much. This means that, whenever the running time is not an important 
factor for the algorithm we can use different base classifiers for each learner or even use one-class classifiers to get better results.

\bibliographystyle{IEEEtran}
\bibliography{crwm}

\end{document}